\documentclass{article}
\usepackage{ijcai18}

\usepackage{times}
\usepackage{xcolor}
\usepackage{soul}
\usepackage[utf8]{inputenc}
\usepackage[font=small]{caption}

\makeatletter
\newcommand\notsotiny{\@setfontsize\notsotiny{8}{9}}
\makeatother

\usepackage{amssymb}
\usepackage{amsmath}
\usepackage{bm}

\usepackage{amsthm}
\usepackage{bbm}
\usepackage{graphicx}
\usepackage{subcaption}
\usepackage{url}
\usepackage{threeparttable}
\usepackage{multirow}
\usepackage{relsize}
\usepackage[linesnumbered,vlined,ruled]{algorithm2e}
\SetAlFnt{\notsotiny}
\SetAlCapFnt{\normalfont\small}
\SetAlCapNameFnt{\small}
\SetArgSty{textrm}
\SetInd{0.1em}{0.6em}
\usepackage{setspace}
\SetKwProg{Fn}{Function}{}{}
\usepackage{tabularx, booktabs}

\usepackage{xpatch}

\makeatletter
\def\thm@space@setup{%
  \thm@preskip=\parskip \thm@postskip=0pt
}
\makeatother

\theoremstyle{plain}

\newtheorem{thm}{Theorem}
\newtheorem{lem}{Lemma}

\theoremstyle{definition}

\makeatletter
\renewenvironment{proof}[1][\proofname]{\par
  \pushQED{\qed}%
  \normalfont
  \topsep0pt \partopsep0pt % no space before
  \trivlist
  \item[\hskip\labelsep
        \itshape
    #1\@addpunct{.}]\ignorespaces
}{%
  \popQED\endtrivlist\@endpefalse
}
\makeatother

\newcommand{\OPEN} {{\textsc{OPEN}}}

\title{Multi-Agent Path Finding with Deadlines \thanks{The research at the University of Southern California was supported by the National Science Foundation (NSF) under grant numbers 1724392, 1409987, and 1319966 as well as a gift from Amazon.
The research at Ben-Gurion University was supported by Israel Science Foundation under grant number 844/17 and the Israel Ministry of Science.}}

\author{
Hang Ma$^1$,
Glenn Wagner$^2$,
Ariel Felner$^3$,
Jiaoyang Li$^1$,
T. K. Satish Kumar$^1$,
Sven Koenig$^1$
\\
$^1$ University of Southern California\\
$^2$ CSIRO\\
$^3$ Ben-Gurion University\\
hangma@usc.edu, glenn.wagner@data61.csiro.au, felner@bgu.ac.il\\ jiaoyanl@usc.edu, tkskwork@gmail.com, skoenig@usc.edu
}
\begin{document}

\maketitle

\providecommand{\agent}[1]{a_{#1}} %Represents a specified robot
\providecommand{\config}{q}
\providecommand{\tstep}{t}
\providecommand{\tcrit}{T_{\emph{end}}}
\providecommand{\dtnode}[1]{\ensuremath{N_{#1}}}
\providecommand{\nodedead}[1]{\dtnode{#1}.\emph{dead}}
\providecommand{\nodelive}[1]{\dtnode{#1}.\emph{live}}
\providecommand{\nodecost}[1]{\dtnode{#1}.\emph{cost}}
\providecommand{\group}{\gamma}

\begin{abstract}  % put your abstract here!
We formalize Multi-Agent Path Finding with Deadlines (MAPF-DL). The objective is to maximize the number of agents that can reach their given goal vertices from their given start vertices within the deadline, without colliding with each other. We first show that MAPF-DL is NP-hard to solve optimally. We then present two classes of optimal algorithms, one based on a reduction of MAPF-DL to a flow problem and a subsequent compact integer linear programming formulation of the resulting reduced abstracted multi-commodity flow network and the other one based on novel combinatorial search algorithms. Our empirical results demonstrate that these MAPF-DL solvers scale well and each one dominates the other ones in different scenarios.
\end{abstract}

\section{Introduction}

In many robotics applications, for example, aircraft-towing vehicles \cite{airporttug16}, warehouse and office robots \cite{kiva,DBLP:conf/ijcai/VelosoBCR15}, game characters \cite{MaAIIDE17}, and other multi-robot systems \cite{MaIEEE17}, robots need to finish tasks that have deadlines. For example, in applications that require long-term autonomy for a team of robots, it is important to move as many robots as possible from a dangerous area to reach a shelter area before a disaster occurs in inclement or adversarial conditions.

One aspect of the problem, namely Multi-Agent Path Finding (MAPF), is to plan collision-free paths for multiple agents in known environments from their given start vertices to their given goal vertices \cite{MaAIMATTERS17}. The objective is to minimize the sum of the arrival times of the agents or the
makespan. MAPF is NP-hard to solve optimally \cite{YuLav13AAAI}
and even to approximate within a small constant factor for makespan
minimization \cite{MaAAAI16}. It can be solved with reductions to other
well-studied combinatorial problems
\cite{Surynek15,DBLP:conf/ecai/SurynekFSB16,YuLav13ICRA,erdem2013general} and dedicated optimal
\cite{ODA11,EPEJAIR,DBLP:journals/ai/SharonSGF13,MStar,DBLP:journals/ai/SharonSFS15,ICBS,FelnerICAPS18},
bounded-suboptimal \cite{ECBS,CohenUK16}, and suboptimal MAPF algorithms
\cite{WHCA,WHCA06,WangB11,PushAndSwap,PushAndRotate} as described in several surveys \cite{MaWOMPF16,SoCS2017Surv}.
MAPF has recently been generalized in different directions
\cite{MaAAMAS16,HoenigICAPS16,MaWOMPF16,HoenigIROS16,MaAAAI17,MaAAMAS17} but none of them capture an important characteristic of many applications, namely
the ability to meet deadlines. A MAPF variant, G-TAPF, assigns tasks with deadlines to agents but does not directly maximize the number of agents that can finish the tasks by the deadlines \cite{GTAPF}.

We thus formalize Multi-Agent Path Finding with Deadlines (MAPF-DL). The objective is to maximize the
number of agents that can reach their given goal vertices from their given
start vertices within a given deadline, without colliding with each other.
%In previously studied MAPF problems, all agents have to be routed from their start vertices to their goal vertices, and the objective is with regard to resources such as fuel (sum of arrival times) or time (makespan). In MAPF-DL, on the other hand, the resources are the agents themselves.
Since none of the existing results directly transfers to MAPF-DL, we first show that MAPF-DL is NP-hard to solve optimally. We then present two families of algorithms to solve MAPF-DL. The first family is based on a reduction of MAPF-DL to a flow problem and a subsequent compact integer linear programming formulation of the resulting reduced abstracted multi-commodity flow network. The second family is based on novel combinatorial search algorithms. We introduce three search-based MAPF-DL algorithms and conduct systematic experiments to compare them on a number of MAPF-DL instances. The results show that all algorithms scale well to large problem instances but each one dominates the other ones in different scenarios. We study their pros and cons and provide a set of guidelines for identifying when each one should be used.

\section{Multi-Agent Path Finding with Deadlines}

In this section, we define MAPF-DL formally and prove its computational hardness. We then present an optimal MAPF-DL algorithm based on integer linear programming (ILP).

\subsubsection{Problem Definition}

We formalize MAPF-DL as follows: We are given a \emph{deadline},
denoted by a time step $\tcrit$, a finite undirected graph $G = (V,E)$, and $M$
agents $\agent{1}, \agent{2} \ldots \agent{M}$. Each agent $\agent{i}$ has a
start vertex $s_i$ and a goal vertex $g_i$. In each time step, each agent
either moves to an adjacent vertex or stays at the same vertex. Each agent can reach its goal vertex in $\tcrit$ time steps in the absence of other agents (without loss of generality). Let $l_i(t)$ be the vertex occupied by agent $\agent{i}$ at time step $t \in
\{0\ldots\tcrit\}$. We call an agent $a_i$ \emph{successful} iff it occupies its goal vertex at the deadline $\tcrit$, that is, $l_i(\tcrit) = g_i$. A
\emph{plan} consists of a path $l_i$ assigned to each successful agent
$\agent{i}$ that satisfies the following conditions:
(1) $l_i(0) = s_i$ [each successful agent starts at its start vertex]. (2) $(l_i(t - 1), l_i(t))\in E$ or $l_i(t - 1) = l_i(t)$ [each successful agent always either moves to an adjacent vertex or does not move]. Each unsuccessful agent $\agent{i}$ is removed at time step zero, and the plan thus contains no path assigned to it, that is, $l_i = \emptyset$.\footnote{Depending on the application, the unsuccessful agents can be removed at time step zero, wait at their start vertices, or move out of the way of the successful agents. We choose the first option in this paper. If the unsuccessful agents are not removed, they can obstruct other agents. However, our proof of NP-hardness does not depend on this assumption, and our MAPF-DL algorithms can be adapted to other assumptions.} We define a \emph{collision} between two different successful agents $\agent{i}$ and $\agent{j}$ to be either a \emph{vertex collision} ($\agent{i}$, $\agent{j}$, $v$, $t$) iff $v = l_i(t) = l_j(t)$ [both successful agents occupy the same vertex simultaneously] or an \emph{edge collision} ($\agent{i}$, $\agent{j}$, $u$, $v$, $t$) iff $u = l_i(t) = l_j(t+1)$ and $v = l_j(t) = l_i(t+1)$ [both successful agents traverse the same edge simultaneously in opposite directions]. A solution is a plan without collisions.

The objective of MAPF-DL is to maximize the number of successful agents $M_{\emph{succ}} = \left\vert\{ \agent{i} | l_i(\tcrit) = g_i\}\right\vert$, that is, the number of paths in the solution, or, equivalently, minimize the number of unsuccessful agents $M_{\emph{unsucc}} = M - M_{\emph{succ}}$. The cost of a plan is thus the number of unsuccessful agents $M_{\emph{unsucc}}$. It can also be defined as the sum of the costs of all agents since a Boolean cost can be defined for each agent where each successful agent incurs cost 0 and each unsuccessful agent incurs cost 1. Obviously, every MAPF-DL instance has a trivial solution where all agents are unsuccessful, namely with cost $M_{\emph{unsucc}} = M$.

%Figure~\ref{fig:example} shows a MAPF-DL instance with $\tcrit = 2$ and agents $\{\agent{1},\agent{2}\}$ with start
%vertices $s_1$ and $s_2$ and goal vertices $g_1$ and $g_2$, respectively. The
%maximum number of escaped agents $M_{\emph{succ}} = 1$ because only $\agent{2}$ can
%reach its goal vertex $g_2$ in two time steps.
%
%\begin{figure}[]
%  \centering
%  \includegraphics[width=0.35\columnwidth]{example.pdf}
%  \caption{A running example of a MAPF-DL instance.}\label{fig:example}
%\end{figure}

\subsubsection{Intractability}

\begin{thm}
It is NP-hard to compute a MAPF-DL solution with the maximum number of successful agents.
\end{thm}

The proof of the theorem reduces the \textbf{$\bm\le$3,$\bm=$3-SAT} problem
\cite{cat1984}, an NP-complete version of the Boolean satisfiability problem,
to MAPF-DL. The reduction is similar to the one used for proving
the NP-hardness of approximating the optimal makespan for MAPF
\cite{MaAAAI16}. It constructs a MAPF-DL instance with deadline $\tcrit = 3$
that has a zero-cost solution iff the given
\textbf{$\bm\le$3,$\bm=$3-SAT} instance is satisfiable. Also see our preliminary work \cite{MaAAMAS18}.

\subsubsection{ILP-Based MAPF-DL Algorithm}

Our ILP-based MAPF-DL algorithm first reduces MAPF-DL to the maximum (integer) multi-commodity flow problem, which is similar to the reductions of MAPF and a MAPF variant, TAPF, to multi-commodity flow problems
\cite{YuLav13ICRA,MaAAMAS16}. It then encodes the latter problem using a compact integer linear programming (ILP) formulation on a reduced abstracted multi-commodity flow network. See our preliminary work \cite{MaAAMAS18} for more details on this algorithm.

\section{Search-Based MAPF-DL Algorithms}

%Combinatorial search is a broad class of algorithms that include the well known
%Djikstra's and A* algorithims for exploring discrete spaces, and which has been
%applied to numerous problems including task scheduling and robotic motion
%planning.  Finding minimal cost solutions to the Multi-Agent Path Finding
%(MAPF) problem is known to be NP-hard \cite{YuLav13AAAI}.  However, recent work
%has produced several combinatorial search algorithms that can find optimal or
%near optimal solutions to problems including large numbers of agents \cite{ODA,
%DBLP:conf/iros/WagnerC11, EPEJAIR, DBLP:journals/ai/SharonSGF13,
%DBLP:journals/ai/SharonSFS15, DBLP:conf/socs/CohenUK15, MStar, WagnerSoCS12,
%ECBS, ICBS}.

In this section, we present a spectrum of optimal combinatorial search algorithms for solving MAPF-DL: Conflict-Based Search with Deadlines (CBS-DL), an adapted version of Conflict-Based Search (CBS) \cite{DBLP:journals/ai/SharonSFS15}; Death-Based Search
(DBS), which reasons about sets of successful agents; and Meta-Agent DBS (MA-DBS), which incorporates the advantages of CBS-DL and DBS.

\subsection{CBS-DL}
\label{sec:CBS}

\begin{algorithm}[t]
\setlength\hsize{8.3cm}
\setstretch{0}
\KwIn{MAPF-DL instance}
    $\emph{Root}.\emph{constraints}\gets \emptyset$ \label{line:CBS:root_constraints}\\
    $\emph{Root}.\emph{plan}\gets$ path for each agent found by a low-level search\label{line:CBS:root_plan}\\
    $\emph{Root}.\emph{cost}\gets 0$ \label{line:CBS:root_key}\\
    $\OPEN \gets \{\emph{Root}\}$\\
    \While{\emph{true}\label{line:CBS:while_statement}}{
        $N \gets \arg\min_{N'\in\OPEN}N'.\emph{cost}$ \label{line:CBS:choose_node}\\
        $\OPEN \gets \OPEN\setminus\{N\}$\\
        Try to find a collision in $N.\emph{plan}$ \label{line:CBS:check_collisions}\\
        \If{$N.\emph{plan}$ has no collision}
        {
            \Return $N.\emph{plan}$ \label{line:CBS:return_solution}
        }
        C $\gets$ first vertex or edge collision $(\agent{i},\agent{j},\dots)$ in $N.\emph{plan}$  \label{line:CBS:new_collision}\\
        {\color{red}
        \emph{// begin: below for MA-DBS only} \label{line:CBS:begin_MADBS}\\
        \If{$\emph{shouldMerge}(\agent{i},\agent{j})$ \label{line:CBS:should_merge}}{
            $\agent{\{i,j\}} \gets \emph{merge}(\agent{i},\agent{j})$ \label{line:CBS:merge_agents}\\
            Update $N.\emph{constraints}$ (external constraints of $\agent{\{i,j\}}$) \label{line:CBS:meta_constraints}\\
            Update $N.\emph{plan}$ by invoking a low-level search (with DBS) for $\agent{\{i,j\}}$ \label{line:CBS:meta_path}\\
            $N.\emph{cost} \gets M_{\emph{unsucc}}$ in $N.\emph{plan}$ \label{line:CBS:meta_key}\\
            $\OPEN \gets \OPEN \cup \{N\}$ \label{line:CBS:insert_node_back}\\
            continue \label{line:CBS:continue_while_statement}// go to [Line \ref{line:CBS:while_statement}]\\
        }
        \emph{// end: above for MA-DBS only} \label{line:CBS:end_MADBS}\\
        }
        \ForEach{$\agent{i}$ involved in $C$ \label{line:CBS:two_child_nodes}}
        {
            $N' \gets$ new node \label{line:CBS:new_child}\\            $N'.\emph{plan} \gets N.\emph{plan}$ \label{line:CBS:inherit_plan}\\
            $N'.\emph{constraints} \gets N.\emph{constraints} \cup \{(\agent{i},\dots)\}$ \label{line:CBS:add_constraints}\\
            Update $N'.\emph{plan}$ by invoking a low-level search (with A* {\color{red}or DBS}) for $\agent{i}$ \label{line:CBS:call_low_level}\\
            $N'.\emph{cost} \gets M_{\emph{unsucc}}$ in $N'.\emph{plan}$\label{line:CBS:cost}\\
            $\OPEN \gets \OPEN \cup \{N'\}$ \label{line:CBS:insert_node}\\
        }
    }
\caption{High Level of CBS-DL ({\color{red}and MA-DBS})} \label{alg:CBS-high}
\end{algorithm}

(Standard) CBS is a two-level MAPF algorithm that minimizes the sum of the arrival times of all agents at their goal vertices. CBS-DL is an adaptation of CBS for MAPF-DL. Algorithm \ref{alg:CBS-high} shows its high-level search. Lines in red are used in MA-DBS (presented in Section \ref{sec:MA-DBS}) only. CBS-DL uses the same framework as CBS but uses $M_{\emph{unsucc}}$ as cost. On the high level, CBS-DL performs a best-first search to resolve collisions among the agents and thus builds a constraint tree (CT). Each CT node contains a set of constraints and a plan that obeys these constraints. CBS-DL always expands the CT node with the smallest cost $M_{\emph{unsucc}}$ of its plan. The root CT node has no constraints [Line \ref{line:CBS:root_constraints}]. CBS-DL performs a low-level search to find a path for each agent (without any constraints). The plan of the root CT node thus contains paths for all agents [Line \ref{line:CBS:root_plan}], and its cost is zero [Line \ref{line:CBS:root_key}]. When CBS-DL expands a CT node $N$, it checks whether the CT node contains a plan that has no collisions [Line \ref{line:CBS:check_collisions}]. If this is the case, $N$ is a goal node and CBS-DL terminates successfully [Line \ref{line:CBS:return_solution}]. Otherwise, CBS-DL chooses a collision to resolve [Line \ref{line:CBS:new_collision}] and generates two child nodes $N_1$ and $N_2$ that inherit all constraints and the plan from $N$ [Line \ref{line:CBS:two_child_nodes}-\ref{line:CBS:inherit_plan}]. If the collision to resolve is a vertex collision $( \agent{i}, \agent{j}, v, t)$, CBS-DL adds the vertex constraint $( \agent{i}, v, t)$ to $N_1$ to prohibit agent $\agent{i}$ from occupying $v$ at time step $t$ and similarly adds the vertex constraint $( \agent{j}, v, t)$ to $N_2$. If the collision to resolve is an edge collision $( \agent{i}, \agent{j}, u, v, t)$, CBS-DL adds the edge constraint $( \agent{i}, u, v, t)$ to $N_1$ to prohibit agent $\agent{i}$ from moving from $u$ to $v$ at time step $t$ and similarly adds the edge constraint $( \agent{j}, v, u, t)$ to $N_2$ [Line \ref{line:CBS:add_constraints}]. For each child CT node, say $N_1$, CBS-DL performs a low-level search for agent $\agent{i}$ to compute a new path from its start vertex to its goal vertex within deadline $\tcrit$ that obeys the constraints of $N_1$ relevant to agent $\agent{i}$ and replaces the old path of agent $\agent{i}$ in $N_1.\emph{plan}$ with the new path returned by the low-level search (it deletes the old path if no path is returned) [Line \ref{line:CBS:call_low_level}]. CBS-DL thus updates the cost of $N_1$ accordingly and inserts it into OPEN [Lines \ref{line:CBS:cost}-\ref{line:CBS:insert_node}].

On the low level, CBS-DL performs an A* search to find a path for the agent from its start vertex to its goal vertex by pruning all nodes with time step $>\tcrit$. If it finds a path from the start vertex to the goal vertex of length exactly $\tcrit$ time steps that obeys the constraints imposed by the high level, it returns the path for the agent and cost 0. Otherwise, it returns no path and cost 1.

\subsubsection{Theoretical Analysis}

We now prove that CBS-DL is complete and optimal.

%We can then rely on existing proofs that CBS finds an optimal solution to MAPF, to show that CBS-DL also finds an optimal solution to MAPF-DL. Moreover, while CBS is guaranteed to find an optimal solution, if one exists, it runs forever if no solution exists and no upper bound on the solution cost is specified. But the deadline ($\tcrit$) of MAPF-DL ensures that the number of possible states is finite, obviating this issue. We therefore give the following theorem:

\begin{lem} \label{lem:finite_nodes}
CBS-DL generates only finitely many CT nodes.
\end{lem}

\begin{proof}
The constraint added on Line \ref{line:CBS:add_constraints} to a child CT node is different from the constraints of its parent CT node since the paths of its parent CT node do not obey it. The depth of the (binary) CT is finite because all paths are not longer than $\tcrit$ and only finitely many different vertex and edge constraints exist.
\end{proof}

\begin{lem}\label{lem:terminate_finite}
Whenever CBS-DL chooses a CT node on Line \ref{line:CBS:choose_node} and the plan of the node has no collisions, then CBS-DL terminates with a solution with finite cost.
\end{lem}

\begin{proof}
The cost of the CT node is $M_{\emph{unsucc}}$ of its plan, which is bounded by $M$.
\end{proof}

\begin{lem} \label{lem:most_succ_agents}
The plan of a CT node has the largest possible number of paths (one for each successful agent) that obey its constraints.
\end{lem}

\begin{proof}[Proof (by induction)]
The statement holds for the root CT node because its plan contains one path for each agent (since each agent can reach its goal vertex in $\tcrit$ time steps in the absence of other agents). Assume that the statement holds for the parent CT node $N$ of any child CT node $N'$. When CBS-DL updates the plan of $N'$ on Line \ref{line:CBS:call_low_level}, it changes the path for one agent only, say agent $\agent{i}$, by performing a low-level search with the constraints of $N'$ (including the newly added constraint $\langle\agent{i},\dots \rangle$). Therefore, CBS-DL correctly updates the path for agent $\agent{i}$, and the statement holds also for $N'$ due to the induction assumption and the fact that $N'.\emph{plan}$ inherits the paths of all agents different from agent $\agent{i}$ from $N.\emph{plan}$ on Line \ref{line:CBS:inherit_plan}.
\end{proof}

\begin{lem}\label{lem:nondecreasing}
CBS-DL chooses CT nodes on Line \ref{line:CBS:choose_node} in non-decreasing order of their costs.
\end{lem}

\begin{proof}
CBS-DL performs a best-first search and the cost of a parent CT node $N$ is at most the cost of any of its child CT nodes $N'$ since $N'.\emph{plan}$ contains at most as many paths as $N.\emph{plan}$ contains because (1) the plan of a CT node contains the largest possible number of paths (one for each successful agent) that obey its constraints according to Lemma \ref{lem:most_succ_agents}, and thus (2) the set of all plans that obey $N'.\emph{constraints}$ is a subset of the set of all plans that obey $N.\emph{constraints}$ (since $N.\emph{constraints} \subset N'.\emph{constraints}$ due to Line \ref{line:CBS:add_constraints}).
%\MEMO{$N.\emph{constraints}$ is a \textbf{strict subset} of $N'.\emph{constraints}$ because of the additional new constraint, while, consequently, the plans that obey $N.\emph{constraints}$ is a \textbf{subset} of all the plans that obey $N'.\emph{constraints}$. The difference is subtle but may be worth noting here.}
\end{proof}

\begin{lem}\label{lem:key_upper_bound}
The cost of a CT node is at most the cost of any solution that obeys its constraints.
\end{lem}

\begin{proof}
The cost of the CT node is the cost $M_{unsucc}$ of its plan, which in turn is the minimum among the costs of all plans that obey its constraints according to Lemma \ref{lem:most_succ_agents}, which in turn is at most the cost of any solution that obeys its constraints since every solution that obeys its constraints is also a plan that obeys its constraints.
\end{proof}

\begin{thm}\label{thm:CBS-DL}
CBS-DL is complete and optimal.
\end{thm}

\begin{proof}
A solution always exists, for example, where all agents are unsuccessful. Now assume that the cost of an optimal solution is $x$ and, for a proof by contradiction, that CBS-DL does not terminate with a solution of cost $x$. Therefore, whenever CBS-DL chooses a CT node with cost $x$ on Line \ref{line:CBS:choose_node}, its plan has collisions (because otherwise CBS-DL would correctly terminate with a solution of cost $x$ according to Lemma \ref{lem:terminate_finite} since it chooses CT nodes on Line \ref{line:CBS:choose_node} in non-decreasing order of their costs according to Lemma \ref{lem:nondecreasing}). Pick an arbitrary optimal solution. A CT node whose constraints the optimal solution obeys has cost $\le x$ according to Lemma \ref{lem:key_upper_bound}. The root CT node is such a node since the optimal solution trivially obeys its (empty) constraints. Whenever CBS-DL chooses such a CT node on Line \ref{line:CBS:choose_node}, its plan has collisions (as shown directly above since its cost is $\le x$). CBS-DL thus generates the child CT nodes of this parent CT node, the constraints of at least one of which the optimal solution obeys and which CBS-DL thus inserts into $\OPEN$ with cost $\le x$. Since CBS-DL chooses CT nodes on Line \ref{line:CBS:choose_node} in non-decreasing order of their costs according to Lemma \ref{lem:nondecreasing}, it chooses infinitely many CT nodes on Line \ref{line:CBS:choose_node} with costs $\le x$, which contradicts Lemma \ref{lem:finite_nodes}.
\end{proof}

\subsection{Death-Based Search}

Death-Based Search (DBS) is also a two-level algorithm. Conceptually, instead of imposing vertex or edge constraints on agents, DBS marks individual agents as unsuccessful and then searches for the minimal set of unsuccessful agents necessary to produce a solution.

We define a group $\group$ of agents to be \emph{consistent} iff all agents in it can simultaneously be successful, that is, the sub-MAPF-DL instance with the agents in $\group$ has a zero-cost solution (an empty group is consistent).  This condition is verified by a special call to CBS-DL with deadline $\tcrit$, which reports that the condition holds if all agents in $\group$ are successful or reports that the condition does not hold once CBS-DL expands a CT node with non-zero cost (that is, at least one agent in $\group$ is not successful).

On the high level, DBS performs a best-first search on the \emph{death tree} (DT).  Each DT node $\dtnode{}$ contains a set \nodelive{} of disjoint groups of live agents (agents that have not been declared unsuccessful) and a cost \nodecost{} equal to the number of agents that have been declared unsuccessful. Algorithm \ref{alg:DBS-high} shows the high-level search of DBS. The root DT node contains a set of $M$ groups of live agents, each group containing a single unique agent [Lines \ref{line:DBS:initiate_set}] and its cost is zero [Line \ref{line:DBS:initiate_key}]. DBS chooses the DT node \dtnode{} with the smallest cost \nodecost{} and checks if all groups in its set \nodelive{} are consistent [Line \ref{line:DBS:choose_node}-\ref{line:DBS:check_consistent}]. If \nodelive{} contains a single consistent group $\group$, the DT node \dtnode{} is a goal node, and DBS returns the zero-cost solution for $\group$ [Line \ref{line:DBS:goal_node}]. If all (more than one) groups in \nodelive{} are consistent, DBS merges the two smallest groups $\group_1$ and $\group_2$ in \nodelive{} to form a new group $\group$ and adds a child DT node whose set contains all the groups in \nodelive{} but replaces $\group_1$ and $\group_2$ with $\group$ [Lines \ref{line:DBS:new_only_child}-\ref{line:DBS:insert_node_back}]. Otherwise, there is an inconsistent group $\group$ in \nodelive{} [Line \ref{line:DBS:inconsistent_group}]. We know that at least
one agent in $\group$ must be declared unsuccessful, forcing a split. In this case, DBS adds $|\group|$ child nodes, one for each agent $\agent{i} \in \group$, to DT, where each of these nodes declares its own unique agent $\agent{i} \in \group$ unsuccessful, and its cost is thus one larger than that of its parent [Lines \ref{line:DBS:new_child}-\ref{line:DBS:insert_child}].

\begin{algorithm}[t]
\setlength\hsize{8.3cm}
\setstretch{0}
\KwIn{MAPF-DL instance}
    $\emph{Root}.\emph{live}\gets \{\{\agent{i}\}|i = 1\ldots M\}$ \label{line:DBS:initiate_set}\\
    $\emph{Root}.\emph{cost}\gets 0$ \label{line:DBS:initiate_key}\\
    $\OPEN \gets \{\emph{Root}\}$ \label{line:DBS:insert_root}\\
    \While{\emph{true}}{
        $N \gets \arg\min_{N'\in\OPEN}N'.\emph{cost}$ \label{line:DBS:choose_node}\\
        $\OPEN \gets \OPEN \setminus \{N\}$ \label{line:DBS:remove_node}\\
         Check whether all groups in $N.\emph{live}$ are consistent by calling CBS-DL \label{line:DBS:check_consistent}\\
        \If{all groups in $N.\emph{live}$ are consistent}
        {
            \If{$|N.\emph{live}| = 1$}
            {
                \Return the zero-cost solution for the single group $\gamma$ in $N.\emph{live}$ \label{line:DBS:goal_node}\\
            }
            \Else
            {
                $N' \gets$ new node \label{line:DBS:new_only_child}\\
                $\gamma \gets \gamma_1 \cup \gamma_2$ ($\gamma_1$ and $\gamma_2$ are the smallest groups in $N.\emph{live}$) \label{line:DBS:two_smallest_groups}\\
                $N'.\emph{live} \gets \left(N.\emph{live} \setminus \{\gamma_1, \gamma_2\} \right) \cup \{\gamma\} $ \label{line:DBS:merge_two_groups}\\
                $N'.\emph{cost}\gets \nodecost{}$ \label{line:DBS:same_key}\\
                $\OPEN \gets \OPEN \cup \{N'\}$ \label{line:DBS:insert_node_back}\\
            }
        }
        \Else
        {
            $\gamma \gets$ first group in $N.\emph{live}$ that does not have a zero-cost solution \label{line:DBS:inconsistent_group}\\
            \ForEach{$\agent{i} \in \gamma$ \label{line:DBS:branch}}
            {
                $N' \gets$ new node \label{line:DBS:new_child}\\
                $N'.\emph{live} \gets \left( \nodelive{}  \setminus \{\group\} \right) \cup \{\group \setminus \{\agent{i}\}\}$ \label{line:DBS:kill_agent}\\
                $N'.\emph{cost} \gets \nodecost{} + 1$ \label{line:DBS:increase key}\\
                $\OPEN \gets \OPEN \cup \{N'\}$ \label{line:DBS:insert_child}\\
            }
        }
    }
\caption{High Level of DBS} \label{alg:DBS-high}
\end{algorithm}

%DBS can be shown to find an optimal solution to MAPF-DL.  Clearly, if DBS never
%marks an agent as dead, DBS requires at most $M-1$ iterations to merge all
%agents into a single group, at which point it recovers a solution.  DBS only marks an agent as dead if it is part of a group $\group$ for which no
%zero-cost solution exists. In this case, no global solution exists where every agent in $\group$ escapes.  Because DBS marks each agent in $\group$ as dead in a separate DT node, the escaped robots in the optimal solution must be a subset of the live agents in at least one DT node, i.e. no valid solutions are lost.
%There are a finite number of possible DT nodes, ensuring that DBS finishes searching the DT in finite time.  Since DBS searches the DT in a best first
%manner, the first solution found is of minimal cost. We
%therefore give the following theorem:

\subsubsection{Other Versions of DBS}

DBS could have started with a root DT node [Line \ref{line:DBS:initiate_set}] whose set contains only a single group of all $M$ agents, which does not require merging groups of live agents but results in a larger branching factor for the root DT node. DBS could have chosen different groups to merge [Line \ref{line:DBS:merge_two_groups}], which might result in an inconsistent group of larger size. Whenever DBS splits a parent DT node [Lines \ref{line:DBS:new_child}-\ref{line:DBS:insert_child}], it could have generated child DT nodes whose sets contain only consistent additional groups (and thus possibly declare more than one additional agent unsuccessful for the child DT nodes), which requires a procedure that can determine all consistent subgroups of the (inconsistent) group $\group$ of agents efficiently and might result in a larger branching factor. In this paper, we chose to present the version that is the easiest to understand and analyze.

\subsubsection{Theoretical Analysis}

We now prove that DBS is complete and optimal.

\begin{lem} \label{lem:DBS:finite_nodes}
DBS generates only finitely many DT nodes.
\end{lem}

\begin{proof}
The branching factor of a DT node is bounded by $M$ due to Line \ref{line:DBS:branch}. Due to Lines \ref{line:DBS:merge_two_groups} and \ref{line:DBS:kill_agent}, when we consider each DT node in a downward traversal of any branch of DT from the root DT node, its set contains either one less group (when merging two groups) or one less agent (when declaring an unsuccessful agent) than that of its parent CT node. Its set is thus different from the sets of all its ancestor DT nodes. Therefore, the depth of DT is also finite since there are finitely many possible sets of disjoints groups of the $M$ agents.
\end{proof}

\begin{lem}\label{lem:DBS:terminate_finite}
Whenever DBS chooses a DT node on Line \ref{line:DBS:choose_node} whose set contains one single consistent group of live agents, then DBS correctly terminates with a solution of finite cost.
\end{lem}

\begin{proof}
Its cost is the number of agents that have been declared unsuccessful, which is bounded by $M$.
\end{proof}

\begin{lem}\label{lem:DBS:nondecreasing}
DBS chooses DT nodes on Line \ref{line:DBS:choose_node} in non-decreasing order of their costs.
\end{lem}

\begin{proof}
DBS performs a best-first search, and the cost of a parent DT node is at most the cost of any of its child DT nodes due to Lines \ref{line:DBS:same_key} and \ref{line:DBS:increase key}.
\end{proof}

%\begin{lem}\label{lem:DBS:key_upper_bound}
%The cost of a DT node is at most the cost of any solution where the agents that have been declared unsuccessful in the DT node are a subset of the unsuccessful agents in the solution.
%\end{lem}
%
%\begin{proof}
%Trivially, the cost is the number of unsuccessful agents.
%\end{proof}

\begin{thm}
DBS is complete and optimal.
\end{thm}

\begin{proof}
A solution always exists, for example, where all agents are unsuccessful. Now assume that the cost of an optimal solution is $x$ and, for a proof by contradiction, that DBS does not terminate with a solution of cost $x$. Therefore, whenever DBS chooses a DT node with cost $x$ on Line \ref{line:DBS:choose_node}, its set does not contain one single consistent group (because otherwise DBS would correctly terminate with a solution of cost $x$ according to Lemma \ref{lem:DBS:terminate_finite} since it chooses DT nodes on Line \ref{line:DBS:choose_node} in non-decreasing order of their costs according to Lemma \ref{lem:DBS:nondecreasing}). Pick an arbitrary optimal solution with the set $\gamma_{\emph{unsucc}}$ of $x$ unsuccessful agents. Trivially, a DT node that has declared the agents in a subset of $\gamma_{\emph{unsucc}}$ unsuccessful has cost $\le x$. The root DT node is such a node since it has not declared any agents unsuccessful. Whenever DBS chooses such a DT node on Line \ref{line:DBS:choose_node}, its set does not contain one single consistent group (as shown directly above since its cost is $\le x$). Its set thus contains (1) more than one consistent group or (2) an inconsistent group (in which case the DT node has declared the agents in a {\em strict} subset of $\gamma_{\emph{unsucc}}$ unsuccessful). In case (1), DBS thus generates the only child DT node of this parent DT node, which has declared the same agents unsuccessful as the parent DT node and which DBS thus inserts into $\OPEN$ with cost $\le x$. In case (2), DBS thus generates the child DT nodes of this parent DT node, at least one of which has still declared the agents (including one additional agent) in a subset of $\gamma_{\emph{unsucc}}$ unsuccessful and which DBS thus inserts into $\OPEN$ with cost $\le x$. Since DBS chooses DT nodes on Line \ref{line:DBS:choose_node} in non-decreasing order of their costs according to Lemma \ref{lem:DBS:nondecreasing}, it chooses infinitely many DT nodes on Line \ref{line:DBS:choose_node} with costs $\le x$, which contradicts Lemma \ref{lem:DBS:finite_nodes}.
\end{proof}

\subsection{Meta-Agent DBS}
\label{sec:MA-DBS}

CBS may perform poorly when an environment contains many possible, but colliding, paths for
the agents since the size of CT is exponential in the number of collisions resolved. On the other hand, DBS may perform poorly for MAPF-DL if the
conflicting agents are not added to the same group early in the search. We thus combine the power of CBS for weakly coupled agents and the power of DBS
for identifying unsuccessful agents in a tightly coupled subset of agents using the Meta-Agent CBS \cite{DBLP:journals/ai/SharonSFS15} framework, which results in a new optimal MAPF-DL algorithm, called Meta-Agent DBS (MA-DBS).

MA-DBS is a two-level algorithm: It uses the high-level search of CBS-DL on the high level and DBS on the low level. Algorithm \ref{alg:CBS-high} shows its high-level search. MA-DBS is similar to CBS-DL but also keeps track of the number of times collisions between every pair of (simple) agents that it has considered thus far during the search in a conflict matrix $CM$. Before MA-DBS expands a CT node $N$ for the colliding agents $\agent{i}$ and $\agent{j}$, if the number of collisions between the two agents exceeds a user-defined \emph{merge threshold} $B$, MA-DBS merges them into a composite \emph{meta agent} $\agent{\{i, j\}}$. To do so, whenever MA-DBS considers a collision between (meta) agents $\agent{i}$ and $\agent{j}$ [Line \ref{line:CBS:new_collision}], because two simple agents $\agent{k} \in \agent{i}$ and $\agent{k'}\in \agent{j}$ collide, it increases the value of $CM[\{k,k'\}]$ by one. Function $\emph{shouldMerge}(\agent{i}, \agent{j})$ returns \emph{true} iff $\sum_{\agent{k}\in\agent{i}, \agent{k'}\in\agent{j}}CM[\{k,k'\}] > B$ [Line \ref{line:CBS:should_merge}]. Since DBS uses a low-level search that finds a plan for a meta agent without any internal collisions between all (simple) agents in the meta agent, it only needs to store external constraints resulting from (external) collisions between any two (simple) agents in different meta agents. Therefore, if MA-DBS decides to merge $\agent{i}$ and $\agent{j}$ into $\agent{\{i, j\}}$, it updates the constraints $N.\emph{constraints}$ of the CT node $N$ accordingly [Line \ref{line:CBS:meta_constraints}]. It then calls DBS to find new paths (without internal collisions) for all agents in $\agent{\{i, j\}}$ subject to the constraints in $N.\emph{constraints}$ relevant to $\agent{\{i, j\}}$ (by solving a MAPF-DL instance with agents in $\agent{\{i, j\}}$) and updates the plan $N.\emph{plan}$ and cost $N.\emph{cost}$ of the CT node $N$ according to the new paths returned by DBS [Lines \ref{line:CBS:meta_path}-\ref{line:CBS:meta_key}]. Then, instead of
expanding $N$, MA-DBS inserts $N$ back into OPEN [Line \ref{line:CBS:insert_node_back}]. When MA-DBS generates a new child CT node, it also calls DBS to find an optimal solution for a meta agent that obeys the constraints of the child CT node [Line \ref{line:CBS:call_low_level}].

\subsubsection{Theoretical Analysis}

Lemmas \ref{lem:finite_nodes} and \ref{lem:nondecreasing} hold for MA-DBS without change. Since the low-level search of MA-DBS, namely DBS, returns the maximum number of paths for a meta agent that obey the constraints of a CT node, Lemma \ref{lem:most_succ_agents} also holds for MA-DBS because (1), when it updates the plan of a CT node on Line \ref{line:CBS:meta_path}, the resulting plan contains the maximum number of paths for the new meta agent and the original paths of the other agents, and (2), when it updates the plan of a child CT node on Line \ref{line:CBS:call_low_level}, the resulting plan contains the maximum number of paths for the meta agent and inherits paths of other agents from the plan of the parent CT node, and thus the induction argument for the proof of Lemma \ref{lem:most_succ_agents} holds. Consequently, Lemma \ref{lem:nondecreasing} also holds for MA-DBS.

\begin{thm}
MA-DBS is complete and optimal.
\end{thm}

\begin{proof}
MA-DBS only merges two agents into one agent on Line \ref{line:CBS:merge_agents} but never splits any agent. Therefore, MA-DBS does the merge operation [Lines \ref{line:CBS:merge_agents}-\ref{line:CBS:continue_while_statement}] finitely many times (bounded by $M$) for each CT node. The rest of the proof is the same as the proof of Theorem \ref{thm:CBS-DL}.
\end{proof}

\begin{figure}[t]
  \centering
  \includegraphics[width=.5\columnwidth]{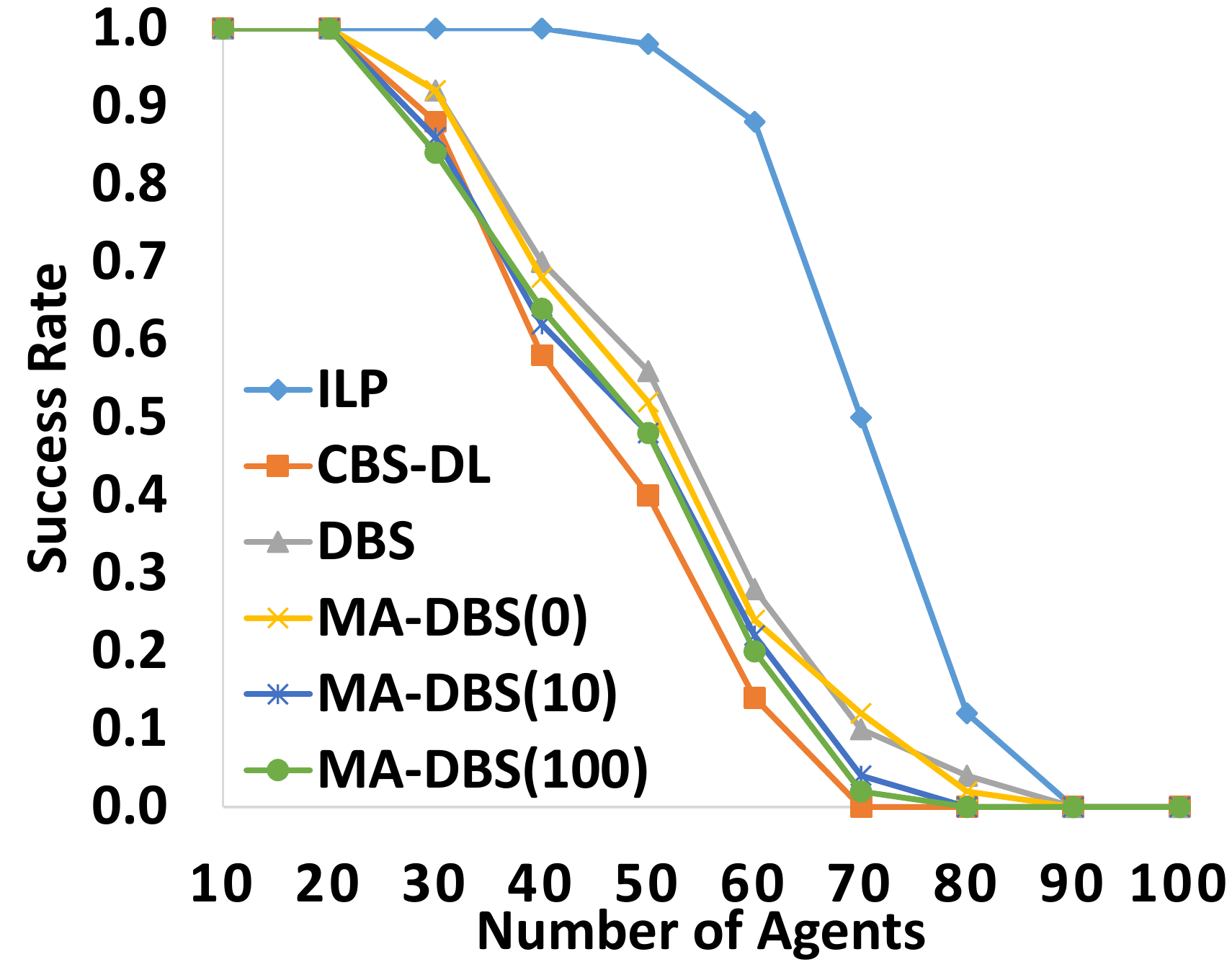}\includegraphics[width=.5\columnwidth]{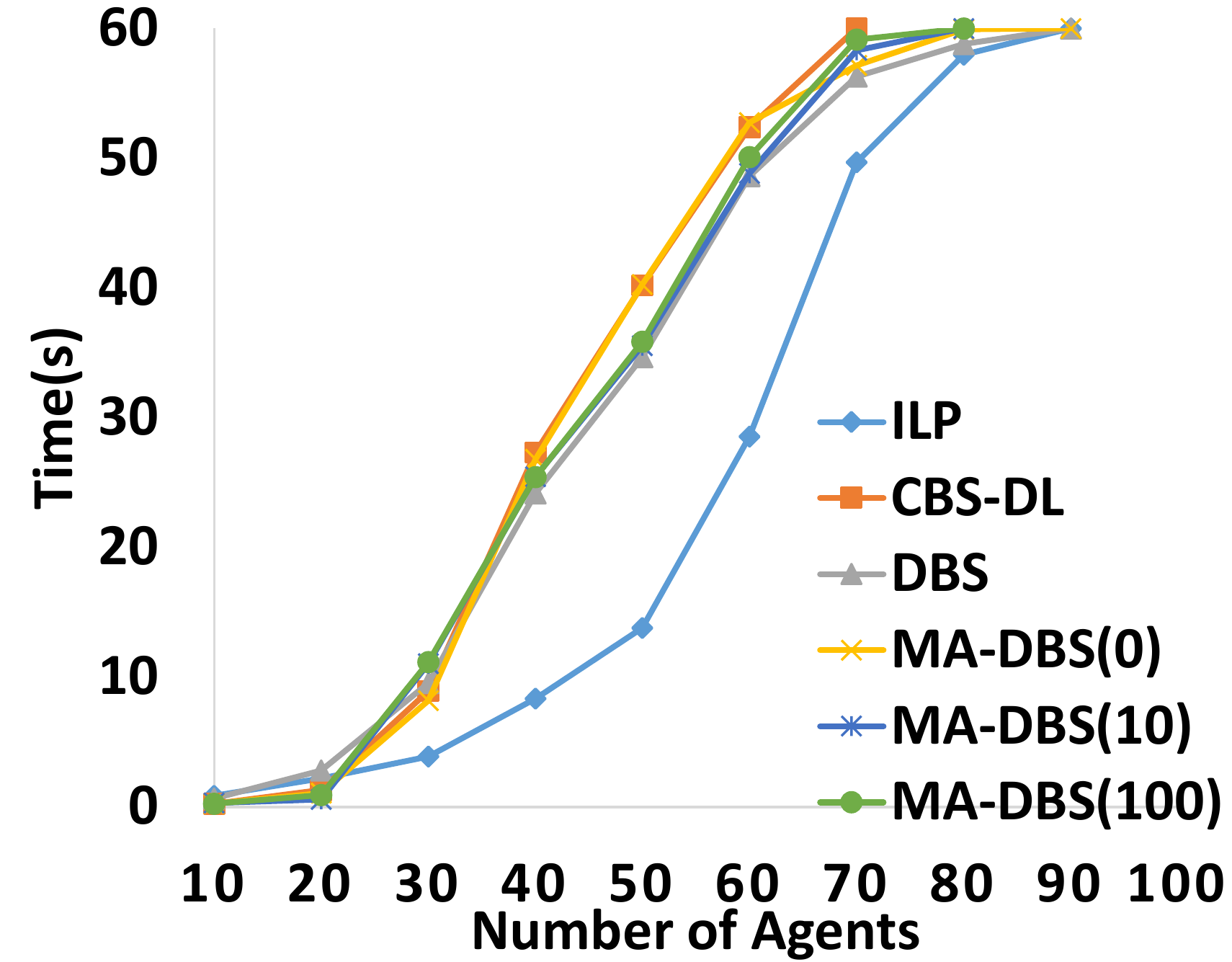}\\
  \Huge
  \resizebox{\columnwidth}{!}{
    \begin{tabular}{c|c|c|c|c|c|c|c}
    \hline
     agents & \#instances & ILP   & CBS-DL & DBS   & MA-DBS(0) & MA-DBS(10) & MA-DBS(100) \\
    \hline
    10    & 50    & 0.86  & \textbf{0.24} & 0.65  & 0.25  & 0.30  & 0.25 \\
    20    & 50    & 2.18  & 1.31  & 2.80  & 1.07  & {\bf 0.58}    & 0.91  \\
    30    & 42    & 3.76  & \textbf{1.57} & 4.24  & 3.23  & 1.76  & 1.85 \\
    40    & 26    & 6.74  & 3.02  & 7.16  & 8.54  & 2.82 & \textbf{2.35} \\
    50    & 14    & 11.88 & 9.23  & 11.76 & 19.43 & \textbf{4.82} & 5.46 \\
    60    & 5     & 29.26 & 4.82  & 10.82 & 27.07 & \textbf{4.77} & 5.88 \\
    \hline
    \end{tabular}%
    }
  \caption{Success rates (top left), averaged running times over all instances (top right), and averaged running times over instances solved by all six algorithms (bottom) for different numbers of agents.}
  \label{fig:S}%
\end{figure}

\section{Experiments}

%\MEMO{AF: There are far too many numbers and charts in this draft. I would
%recommend:
%
%1) Deleting the running time from the (tab:50, tab:100 ...) tables and merging
%these tables into one large table with success rates.
%
%2) I do not see a reason to have both a table for success rate and a chart. I
%would pick one of them and stick to it.
%
%3) You then have two sources of information on the running time. You may choose
%to delete one of them or keep them both. They are both not ideal because we
%stopped after 60 seconds.}

In this section, we describe our experimental results on a 2.50 GHz Intel Core
i5-2450M laptop with 6 GB RAM. We tested six optimal MAPF-DL
algorithms: the ILP-based algorithm, CBS-DL, DBS, and MA-DBS with merge
thresholds 0, 10, and 100 (labeled as MA-DBS(0), MA-DBS(10), and MA-DBS(100), respectively). The ILP-based algorithm uses CPLEX V12.7.1 \cite{IBM2011} as the ILP solver. We experimented on instances where the start and goal vertices of each agent are placed randomly so that the distance between them is close to the deadline. An instance becomes much easier to solve if this distance is much smaller than the deadline (since there is more leeway to plan a path for the agent). Specifically, we use three sets of randomly generated MAPF-DL instances with different numbers of agents (varied from 10 to 100 in increments of 10) labeled as SMALL, MEDIUM, and LARGE on $40\times40$, $80\times80$, and $120\times120$ 4-neighbor 2-D grids with deadlines $\tcrit =$ 50, 100, and 150, respectively. The cells in each grid are blocked independently at random with 20\% probability each. We generate 50 MAPF-DL instances for each number of agents for each set. The start and goal vertices of each agent are randomly placed at distance 48, 49, or 50 for SMALL, 98, 99, or 100 for MEDIUM, and 148, 149, or 150 for LARGE. Each algorithm is given a time limit of 60 seconds to solve each instance. We did not run an algorithm for some number of agents if it solved none of the 50 instances for a smaller number of agents.

\noindent \textbf{The $40\times40$ SMALL domain}
Figure~\ref{fig:S} (top left) plots the success rates (numbers of instances solved within the time limit divided by 50) for all algorithms for SMALL. ILP has the highest success rates, and they start to drop only at 50 agents. The success rates for the search-based algorithms start to drop at 30 agents. DBS and MA-DBS(0) have the highest success rates among all search-based algorithms.
Figure~\ref{fig:S} (top right) plots the average running times over \textbf{all} 50 instances. 60 seconds are used for an instance that is not solved. Therefore, the data points in the chart are lower bounds on the running times in those cases when not all instances are solved. ILP performs the best. Finally, the table in Figure~\ref{fig:S} reports the average running times over those ``easy'' instances that are \textbf{solved by all six algorithms} (it also reports the numbers of those instances but does not show the rows where no instance is solved by all the algorithms). The best entry in each row is shown in bold. The search-based algorithms use less time to solve these ``easy'' instances than ILP. CBS-DL, MA-DBS(10), and MA-DBS(100) seem to use the least times and outperform ILP by up to a factor of 6. In some cases, the running times are smaller for larger numbers of agents because fewer (and ``easier'') instances are solved by all algorithms.

\begin{figure}[t]
  \centering
  \includegraphics[width=.5\columnwidth]{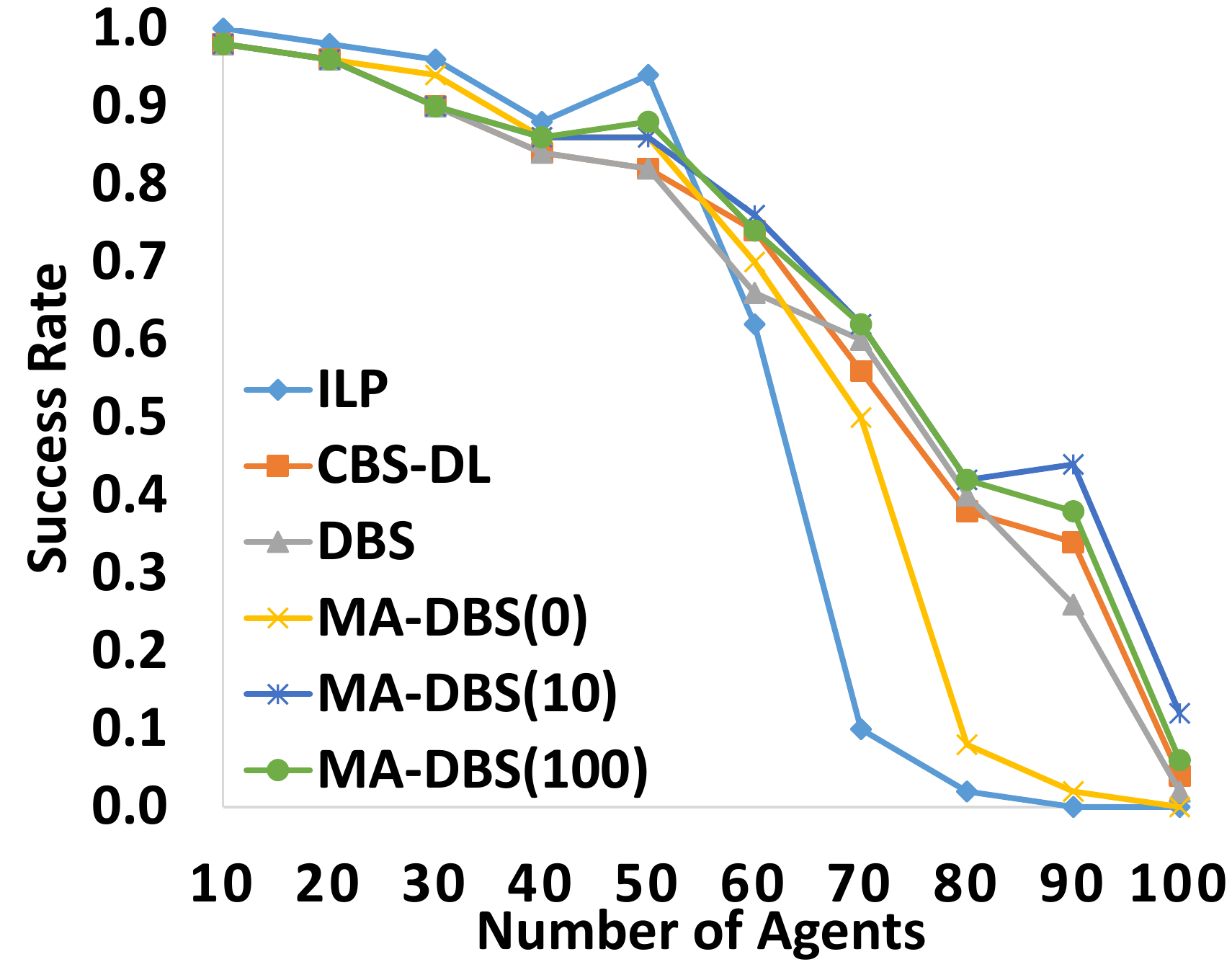}\includegraphics[width=.5\columnwidth]{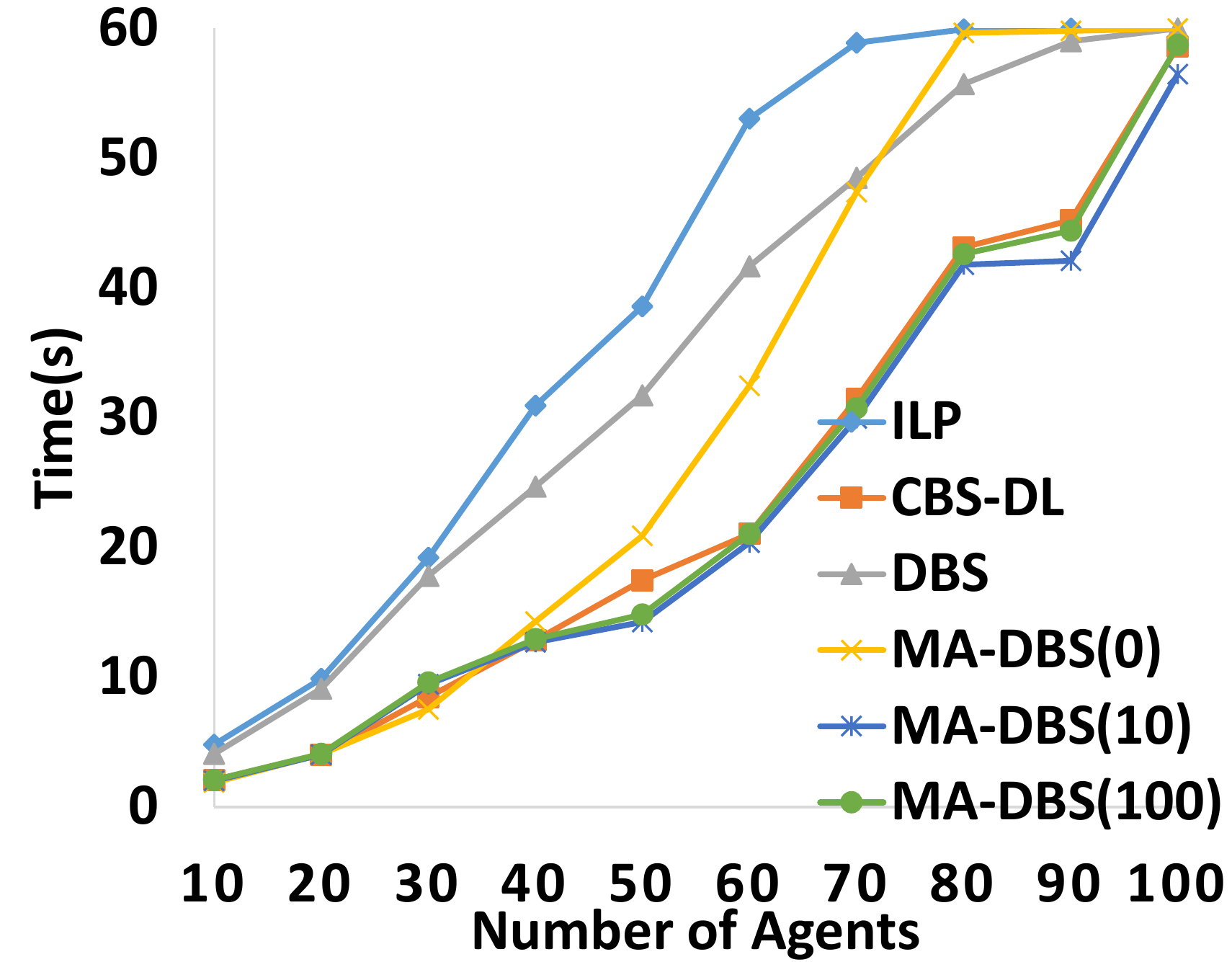}
  \Huge
  \resizebox{\columnwidth}{!}{
    \begin{tabular}{c|c|c|c|c|c|c|c}
    \hline
    agents & \#instances & ILP   & CBS-DL & DBS   & MA-DBS(0) & MA-DBS(10) & MA-DBS(100) \\
    \hline
    10    & 49    & 3.80  & 0.89  & 2.95  & \textbf{0.69} & 0.82  & 0.90 \\
    20    & 48    & 8.87  & \textbf{1.71} & 6.97  & 1.75  & \textbf{1.71} & 1.75 \\
    30    & 43    & 17.56 & \textbf{2.65} & 12.19 & 3.01  & 2.86  & 3.02 \\
    40    & 42    & 26.92 & \textbf{3.84} & 17.96 & 6.04  & 4.08  & 4.24 \\
    50    & 37    & 37.46 & \textbf{5.92} & 25.57 & 14.33 & 6.36  & 6.12 \\
    60    & 22    & 48.80 & \textbf{7.26} & 31.89 & 21.31 & 7.29  & 7.33 \\
    70    & 3     & 47.76 & 8.70  & 35.06 & 26.37 & \textbf{7.74} & 7.97 \\
    \hline
    \end{tabular}%
    }
  \caption{Results for MEDIUM.}
  \label{fig:M}%
\end{figure}

\begin{figure}[t]
  \centering
  \includegraphics[width=.5\columnwidth]{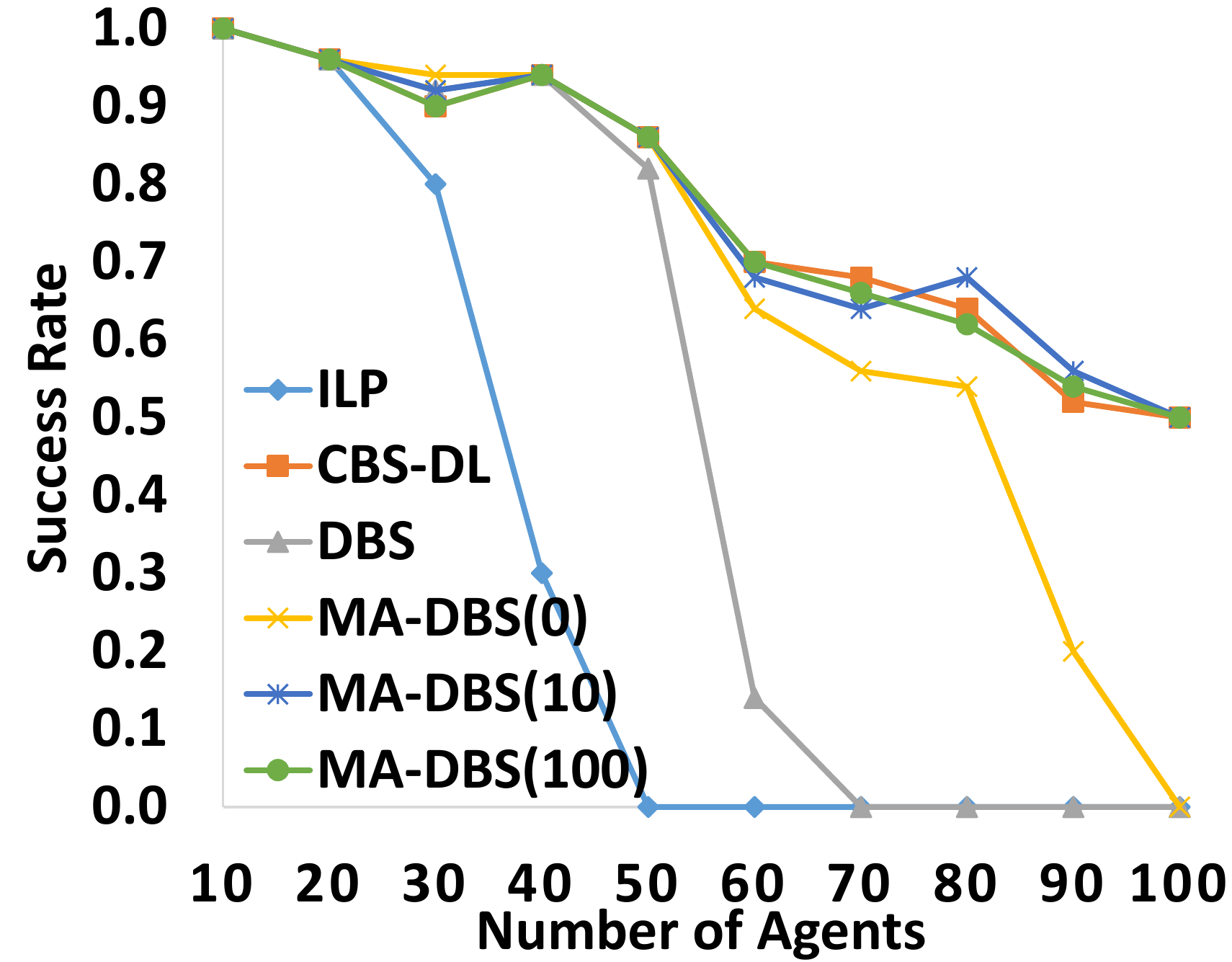}\includegraphics[width=.5\columnwidth]{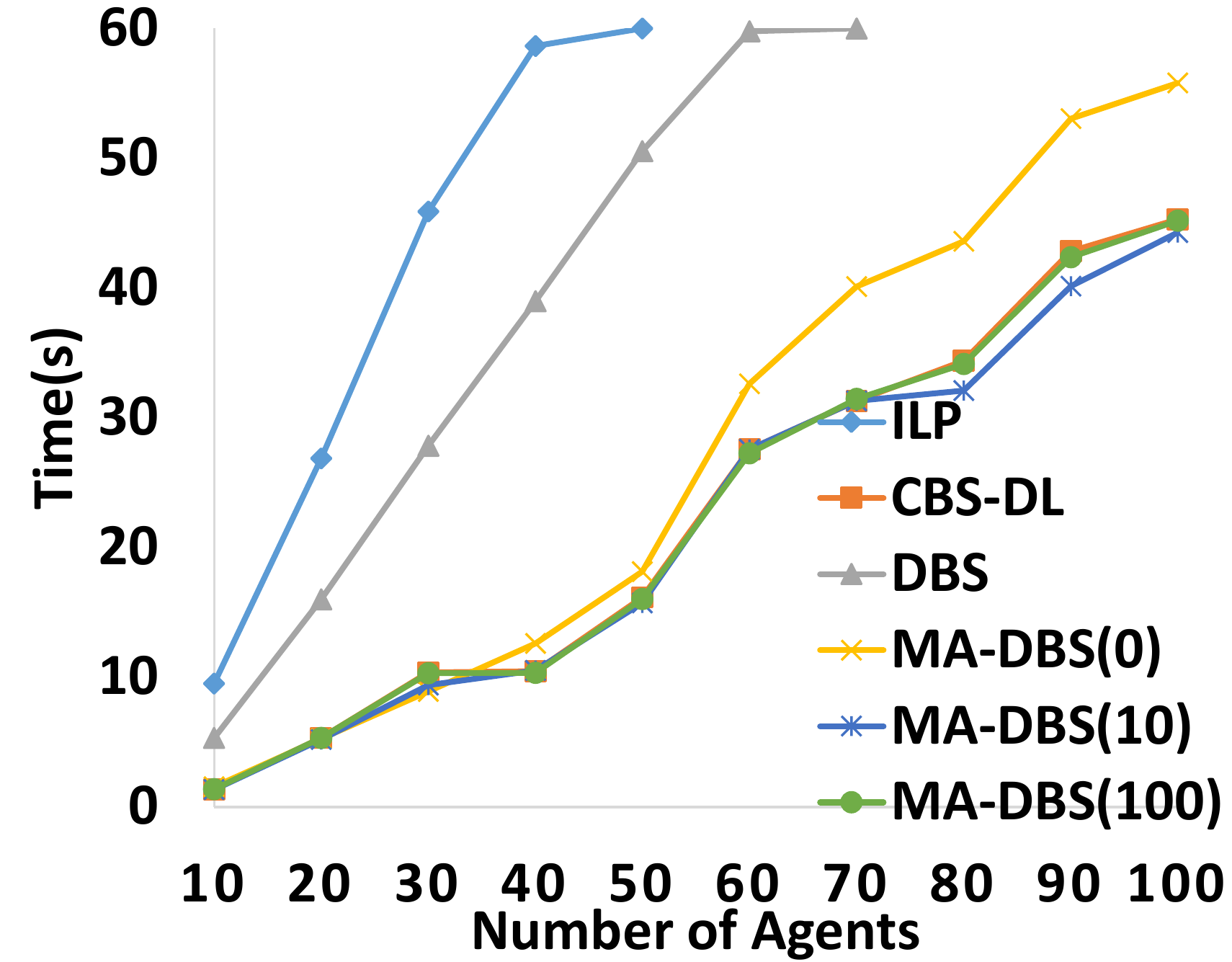}
  \Huge
  \resizebox{\columnwidth}{!}{
    \begin{tabular}{c|c|c|c|c|c|c|c}
    \hline
    agents & \#instances & ILP   & CBS-DL & DBS   & MA-DBS(0) & MA-DBS(10) & MA-DBS(100) \\
    \hline
    10    & 50    & 9.52  & 1.38  & 5.31  & 1.55  & \textbf{1.33} & 1.36 \\
    20    & 48    & 25.49 & 3.04  & 14.15 & 2.98  & \textbf{2.91} & 3.03 \\
    30    & 38    & 42.43 & 4.84  & 24.88 & 5.67  & 4.86  & \textbf{4.80} \\
    40    & 15    & 55.44 & 7.20  & 37.23 & 8.61  & 6.88  & \textbf{6.67} \\
    \hline
    \end{tabular}%
    }
  \caption{Results for LARGE.}
  \label{fig:L}%
\end{figure}

\noindent \textbf{The $80\times80$ MEDIUM domain}
Figure~\ref{fig:M} reports the same statistics for MEDIUM in the same format as reported for SMALL. Figure~\ref{fig:M} (top left) plots the success rates. ILP has the highest success rates for small numbers ($\leq 50$) of agents but the lowest success rates for large numbers of agents. MA-DBS(10) seems to perform the best for large numbers of agents. Figure~\ref{fig:M} (top right) plots the average running times over \textbf{all} 50 instances. ILP has the longest running times. MA-DBS(10) seems to perform the best in general. Finally, the table in Figure~\ref{fig:M} reports the average running times over instances that are \textbf{solved by all six algorithms}. ILP performs the worst. CBS-DL seems to have the smallest running times and outperforms ILP by up to a factor of 7.

\noindent \textbf{The $120\times120$ LARGE domain}
Figure~\ref{fig:L} reports the same statistics for LARGE in the same format as reported for SMALL.
Figure~\ref{fig:L} (top left) plots the success rates. CBS-DL, MA-DBS(10), and MA-DBS(100) have the best success rates. ILP has the worst success rates. Figure~\ref{fig:L} (top right) plots the average running times over
\textbf{all} 50 instances. MA-DBS(10) seems to perform the best. ILP performs the worst. The table in Figure~\ref{fig:L} reports the average
running times over instances that are \textbf{solved by all six algorithms}. MA-DBS(10) and CBS-DL perform the best (very close to each other) and outperform ILP by up to a factor of 9.

\noindent \textbf{Summary of Experimental Results}
For the same number of agents, SMALL has higher agent density, more tightly-coupled agents, and shorter planning horizons than MEDIUM and LARGE. ILP outperforms the search-based algorithms for SMALL because the size of the ILP formulation is small. When $\tcrit$ increases, the size of the ILP formulation and the running time required to solve it increase significantly.

On the other hand, among all search-based algorithms, there seems to be a spectrum where DBS and CBS-DL sit at two extremes. DBS has higher success rates than CBS-DL for SMALL. CBS-DL has significantly higher success rates than DBS for MEDIUM and LARGE. CBS-DL uses much less times than DBS for instances that are solved by all algorithms. MA-DBS seems to balance between CBS-DL and DBS: (a) MA-DBS(0) is more similar to DBS because it merges agents into meta agents more frequently, which can result in a large meta agent containing many agents that need to be solved by DBS on the low level; and, on the other hand, (b) MA-DBS(10) and MA-DBS(100) are more similar to CBS-DL because they merge agents less frequently and their searches mostly remain in the CBS-DL framework.

%To summarize, ILP outperforms the search-based algorithms for high agent density, tightly-coupled agents, and short planning horizons but it performs significantly worse than the search-based algorithms for low agent density, weekly-coupled agents, and long planning horizons. Similar trends are reported for comparing SAT-based algorithms to CBS-based
%algorithms for solving MAPF \cite{DBLP:conf/ecai/SurynekFSB16}.
%Among the search-based algorithms, MA-DBS(10) seems to perform the best in general. MA-DBS incorporates the advantages of both CBS-DL and DBS. It merges tightly-coupled agents early but still splits groups of agents when agents have to be declared unsuccessful. The merge threshold 10 seems to give the best balance for the domains we tested. However, MA-DBS is more difficult to understand and implement than other algorithms and existing work also indicates that it is hard to determine the best merge threshold across different domains~\cite{DBLP:journals/ai/SharonSFS15}. As suggested in that work, we need to explore different merge criteria in future work.

\section{Conclusions and Future Work}

We formalized MAPF-DL, a new variant of MAPF. Theoretically, we proved that MAPF-DL is NP-hard to solve optimally. We presented two families of optimal MAPF-DL algorithms, one based on an ILP formulation and one based on combinatorial search techniques. Our experimental results show that each of them performs the best in different scenarios.
We suggest the following future directions: (1) develop and compare new MAPF-DL algorithms, for example, A*-, ASP-, and SAT-based algorithms; (2) study important generalizations of MAPF-DL (for example, when agents have different priorities) more deeply; (3) study the combinatorial difference between MAPF-DL and MAPF; and (4) explore different merge criteria for MA-DBS.

{
%\fontsize{8.8}{9.8}\selectfont
\small
\bibliographystyle{named}
\bibliography{references}
}

\end{document}